\documentclass[letterpaper, 11pt]{article}

\usepackage[english]{babel}
\usepackage[utf8x]{inputenc}
\usepackage[T1]{fontenc}
\usepackage{chngcntr}
\usepackage{apptools}
\AtAppendix{\counterwithin{lemma}{section}}
\usepackage[margin=1in]{geometry}
\usepackage{graphicx}
\usepackage{subfigure}
\usepackage{booktabs} % for professional tables

\usepackage{hyperref}

\usepackage{algorithm}
\usepackage{algorithmic}

\usepackage{tikz}
\usetikzlibrary{backgrounds,matrix,fit}
\usepackage{empheq}
\usepackage{amsfonts}
\usepackage{amsthm}
\usepackage{bbm}
\usepackage{authblk}

\newtheorem{assumptions}{Assumptions}
\newtheorem{theorem}{Theorem}
\newtheorem{lemma}{Lemma}
\newtheorem{corollary}{Corollary}

\allowdisplaybreaks

\usepackage{enumitem}

\usepackage{xspace}
\newcommand{\ie}{{i.e.},\xspace}

\newcommand{\E}{\mathbb{E}}

\newcommand{\D}{\mathcal{D}}

\title{Learning One Convolutional Layer with Overlapping Patches}
\author[1]{Surbhi Goel\footnote{surbhi@cs.utexas.edu}}
\author[1]{Adam Klivans\footnote{klivans@cs.utexas.edu}}
\author[2]{Raghu Meka\footnote{raghum@cs.ucla.edu}}
\affil[1]{Department of Computer Science, University of Texas at Austin}
\affil[2]{Department of Computer Science, UCLA}
\date{}
\begin{document}
\maketitle

\begin{abstract}
We give the first provably efficient algorithm for learning a one hidden layer convolutional network with respect to a general class of (potentially overlapping) patches.  Additionally, our algorithm requires only mild conditions on the underlying distribution.  We prove that our framework captures commonly used schemes from computer vision, including one-dimensional and two-dimensional ``patch and stride'' convolutions.

% Let ${\cal D}$ be a symmetric distribution and let $\kappa$ be the condition number of the covariance matrix induced by the patch structure.  For the ReLU activation, we properly recover the underlying weight vector in polynomial-time in all relevant parameters, including $\kappa$.  

Our algorithm-- {\em Convotron}-- is inspired by recent work applying isotonic regression to learning neural networks.  Convotron uses a simple, iterative update rule that is stochastic in nature and tolerant to noise (requires only that the conditional mean function is a one layer convolutional network, as opposed to the realizable setting).  In contrast to gradient descent, Convotron requires no special initialization or learning-rate tuning to converge to the global optimum. 

We also point out that learning one hidden convolutional layer with respect to a Gaussian distribution and just {\em one} disjoint patch $P$ (the other patches may be arbitrary) is {\em easy} in the following sense:  Convotron can efficiently recover the hidden weight vector by updating {\em only} in the direction of $P$.  

\end{abstract}
\newpage
\section{Introduction}
\label{sec:intro}
Developing \textit{provably} efficient algorithms for learning commonly used neural network architectures continues to be a core challenge in machine learning. The underlying difficulty arises from the highly non-convex nature of the optimization problems posed by neural networks. Obtaining provable guarantees for learning even very basic architectures remains open.

In this paper we consider a simple convolutional neural network with a single filter and overlapping patches followed by average pooling (Figure \ref{fig:cnn}). More formally, for an input image $x$, we consider $k$ {\em patches} of size $r$ indicated by {\em selection} matrices $P_1, \ldots, P_k \in \{0,1\}^{r \times n}$ where each matrix has exactly one $1$ in each row and at most one $1$ in each column. The neural network is computed as $f_{w}(x)= \frac{1}{k}\sum_{i=1}^k \sigma(w^T  P_i  x)$ where $\sigma$ is the activation function and $w \in \mathbb{R}^r$ is the weight vector corresponding to the convolution filter. We focus on ReLU and leaky ReLU activation functions.

\begin{figure}[ht]
\label{fig:cnn}
\centerline{\includegraphics[width=0.4\textwidth]{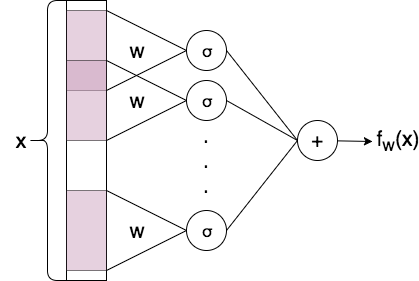}}
\caption{Architecture of convolutional network with one hidden layer and average pooling. Each purple rectangle corresponds to a patch.}
\end{figure}

\subsection{Our Contributions}
The main contribution of this paper is a simple, stochastic update algorithm {\em Convotron} (Algorithm \ref{alg:convotron}) for provably learning the above convolutional architecture. The algorithm has the following properties:
\begin{itemize}
	\item Works for general classes of overlapping patches and requires mild distributional conditions.	
        \item Proper recovery of the unknown weight vector. 
	\item Stochastic in nature with a ``gradient-like'' update step.
	\item Requires no special/random initialization scheme or tuning of the learning rate. 
%	\item Requires no learning rate tuning making it hyperparameter free unlike SGD.
        \item Tolerates noise and succeeds in the {\em probabilistic concept} model of learning.
        \item Logarithmic convergence in $1/\epsilon$, the error parameter, in the realizable setting.
\end{itemize}
This is the first efficient algorithm for learning general classes of overlapping patches (and the first algorithm for any class of patches that succeeds under mild distributional assumptions). Prior work has focused on analyzing SGD in the realizable/noiseless setting with the caveat of requiring either disjoint patches \cite{brutzkus2017globally,du2017gradient} with Gaussian inputs or technical conditions linking the underlying true parameters and the ``closeness of patches'' \cite{du2017convolutional}. 

In contrast, our conditions depend only on the patch structure itself and can be efficiently verified.  Commonly used patch structures in computer vision applications such as 1D/2D grids satisfy our conditions. Additionally, we require only that the underlying distribution on samples is symmetric and induces a covariance matrix on the patches with polynomially bounded condition number\footnote{Brutzkus and Globerson \cite{brutzkus2017globally} proved that the problem, even with disjoint patches, is NP-hard in general, and so some distributional assumption is needed for efficient learning.}. All prior work handles only continuous distributions. Another major difference from prior work is that we give guarantees using purely empirical updates.  That is, we do not require an assumption that we have access to exact quantities such as the population gradient of the loss function. 

We further show that in the commonly studied setting of Gaussian inputs and non-overlapping patches, updating with respect to a single non-overlapping patch is sufficient to guarantee convergence. This indicates that the Gaussian/no-overlap assumption is quite strong.

\subsection{Our Approach}
Our approach is to exploit the monotonicity of the activation function instead of the strong convexity of the loss surface. We use ideas from isotonic regression and extend them in the context of convolutional networks. These ideas have been successful for learning generalized linear models \cite{kakade2011efficient}, improperly learning fully connected, depth-three neural networks \cite{goel2017learning}, and learning graphical models \cite{klivans2017learning}.

%\subsection{Organization}
\subsection{Related Work}
It is known that in the worst case, learning even simple neural networks is computationally intractable.  For example, in the non-realizable (agnostic) setting, it is known that learning a single ReLU (even for bounded distributions and unit norm hidden weight vectors) with respect to square-loss is as hard as learning sparse parity with noise \cite{goel2016reliably}, a notoriously difficult problem from computational learning theory.  For learning one hidden layer convolutional networks, Brutzkus and Globerson \cite{brutzkus2017globally} proved that distribution-free recoverability of the unknown weight vector is NP-hard, even if we restrict to disjoint patch structures.

As such, a major open question is to discover the mildest assumptions that lead to polynomial-time learnability for simple neural networks.  In this paper, we consider the very popular class of convolutional neural networks (for a summary of other recent approaches for learning more general architectures see \cite{goel2017eigenvalue}).  For convolutional networks, all prior research has focused on analyzing conditions under which (Stochastic) Gradient Descent converges to the hidden weight vector in polynomial-time.  

Along these lines, Brutzkus and Globerson \cite{brutzkus2017globally} proved that with respect to the spherical Gaussian distribution and for disjoint (non-overlapping) patch structures, gradient descent recovers the weight vector in polynomial-time.  Zhong et al. \cite{zhong2017a} showed that gradient descent combined with tensor methods can recover one hidden layer involving multiple weight vectors but still require a Gaussian distribution and non-overlapping patches.  Du et al. \cite{du2017gradient} proved that gradient descent recovers a hidden weight vector involved in a type of two-layer convolutional network under the assumption that the distribution is a spherical Gaussian, the patches are disjoint, and the learner has access to the true population gradient of the loss function.  

We specifically highlight the work of Du, Lee, and Tian \cite{du2017convolutional}, who proved that gradient descent recovers a hidden weight vector in a one-layer convolutional network under certain technical conditions that are more general than the Gaussian/no-overlap patch scenario.  Their conditions involve a certain ``alignment'' of the unknown patch structure, the hidden weight vector, and the (continuous) marginal distribution.  However, it is unclear which concrete patch-structure/distributional combinations their framework captures.  We also note that all of the above results assume there is no noise; \ie they work in the realizable setting.

Other related works analyzing gradient descent with respect to the Gaussian distribution (but for non-convolutional networks) include \cite{soltanolkotabi2017learning, ge2017, zhong2017recovery, tian2016symmetry, li2017convergence,zhang2017electron}.

In contrast, we consider an alternative to gradient descent, namely Convotron, that is based on isotonic regression.  The exploration of alternative algorithms to gradient descent is a feature of our work, as it may lead to new algorithms for learning deeper networks.

\section{Preliminaries}

$||\cdot||$ corresponds to the $l_2$ -norm for vectors and the spectral norm for matrices. The identity matrix is denoted by $I$. We denote the input-label distribution by $\D$ over input drawn from $\mathcal{X}$ and label drawn from $\mathcal{Y}$. The marginal distribution on the input is denoted by $\D_{\mathcal{X}}$ and the corresponding probability density function is denoted by $P_{\mathcal{X}}$.

In this paper we consider a simple convolution neural network with one hidden layer and average pooling. Given input $x \in \mathbb{R}^n$, the network computes $k$ patches of size $r$ where each patch's location is indicated by matrices $P_1, \ldots, P_k \in \{0,1\}^{r \times n}$. Each matrix $P_i$ has exactly one 1 in each row and at most one 1 in every column. As before, the neural network is computed as follows:
\[
f_{w}(x)= \frac{1}{k}\sum_{i=1}^k \sigma(w^T  P_i  x)
\]
where $\sigma$ is the activation function and $w \in \mathbb{R}^r$ is the weight vector corresponding to the convolution filter.

We study the problem of learning the teacher network under the square loss, that is, we wish to find a $w$ such that
\begin{align*}
L(w) := \E_{(x,y) \sim \D}[(y - f_{w}(x))^2] \leq \epsilon.
\end{align*}

\begin{assumptions}\label{assume}
We make the following assumptions:
\begin{enumerate}[label=(\alph*)]
	\item \textbf{Learning Model}: Probabilistic Concept Model
          \cite{kearns1990efficient}, that is, for all $(x,y) \sim \D$, $y
          = f_{w_*}(x) + \xi$, for some unknown $w_*$ where $\xi$ is
          noise with $\E[\xi|x] = 0$ and $\E[\xi^4|x] \leq \rho$ for
          some $\rho>0$. Note we do not require that the noise is
          independent of the instance.\footnote{In the realizable setting, as in previous works, it is assumed that $\xi = 0$.}
	\item \textbf{Distribution}: The marginal distribution on the
          input space $\D_\mathcal{X}$ is a symmetric distribution about the origin, that is, for all $x$, $P_{\mathcal{X}}(x) = P_{\mathcal{X}}(-x)$.
	\item \textbf{Patch Structure}: The minimum eigenvalue of $P_{\Sigma} := \sum_{i, j = 1}^{k} P_i \Sigma P_j^T$ where $\Sigma = \E_{x \sim \D_\mathcal{X}}[xx^T]$ and the maximum eigenvalue of $P:= \sum_{i, j = 1}^{k} P_i P_j^T$ are polynomially bounded.
	\item \textbf{Activation Function}: The activation function has the following form:
\[
\sigma(x) = 	\begin{cases}
					x & \text{if } x \geq 0 \\
					\alpha x & \text{ otherwise}
				\end{cases}
\]
for some constant $\alpha \in [0, 1]$.
\end{enumerate}
\end{assumptions}
The distributional assumption includes common assumptions such as
Gaussian inputs, but is far less restrictive.  For example, we do
not require the distribution to be continuous nor do we require it to
have identity covariance.  In Section
\ref{sec:patch}, we show that commonly used patch schemes from
computer vision satisfy our patch requirements. The assumption on
activation functions is satisfied by popular activations such as ReLU
($\alpha = 0$) and leaky ReLU ($\alpha > 0$).

\subsection{Some Useful Properties}

The activations we consider in this paper have the following useful property:

\begin{lemma}\label{lem:switch}
For all $a,b \in \mathbb{R}$,
\[
\E_{x \sim \D}[\sigma(a^T  x)(b^T  x)] = \frac{1 + \alpha}{2}\E_{x \sim \D}[(a^T  x)(b^T  x)].
\]
\end{lemma}

The loss function can be upper bounded by the $l_2$-norm distance of weight vectors using the following lemma.
\begin{lemma}
\label{lem:loss}
For any $w$, we have
\[
L(w) \leq \frac{(1 + \alpha)\lambda_{\max}(\Sigma)}{2}||w_* - w||^2.
\]
\end{lemma}

\begin{lemma}
\label{lem:sq}
For all $w$ and $x$,
\[
(f_{w_*}(x) - f_{w}(x))^2 \leq ||w_* - w||^2||x||^2
\]
\end{lemma}

The Gershgorin Circle Theorem, stated below, is useful for bounding the eigenvalues of matrices.
\begin{theorem}[\cite{weisstein2003gershgorin}]\label{thm:circle}
For a $n \times n$ matrix $A$, define $R_i:= \sum_{j = 1, j \neq i}^n |A_{i,j}|$. Each eigenvalue of $A$ must lie in at least one of the disks $\{z: |z - A_{i,i}| \leq R_i\}$.
\end{theorem}

\noindent\textbf{Note}: The proofs of lemmas in this section have been deferred to the Appendix.
\section{The Convotron Algorithm}

In this section we describe our main algorithm \textit{Convotron} and give a proof of its correctness. Convotron is an iterative algorithm similar in flavor to SGD with a modified (aggressive) gradient update. Unlike SGD (Algorithm \ref{alg:sgd}), Convotron comes with provable guarantees and also does not need a good initialization scheme for convergence.
\begin{algorithm}[tb]
   \caption{Convotron}
   \label{alg:convotron}
\begin{algorithmic}
   \STATE Initialize $w_1 := 0 \in \mathbb{R}^r$.
   \FOR{$t=1$ {\bfseries to} $T$}
   \STATE Draw $(x_t, y_t) \sim \mathcal{D}$
   \STATE Let $G_t = (y_t - f_{w_t}(x_t))  \left(\sum_{i=1}^k P_i  x_t\right)$
   \STATE Set $w_{t+1} = w_t + \eta G_t$
   \ENDFOR
   \STATE {Return $w_{T+1}$}
\end{algorithmic}
\end{algorithm}

The following theorem describes the convergence rate of our algorithm:

\begin{theorem}\label{thm:main}
% If Assumptions \ref{assume} are satisfied then for $\eta = \frac{(1 + \alpha)\lambda_{\min}(P_\Sigma)}{3k\lambda_{\max}(P)}\min \left(\frac{1}{\E_{x}[||x||^4]}, \frac{\epsilon ||w_*||^2}{\sqrt{\rho\E_{x}[||x||^4]}} \right)$ and $T \geq \frac{3k}{\eta(1 + \alpha)\lambda_{\min}(P_\Sigma)} \log\left(\frac{1}{\epsilon \delta}\right)$, with probability $1-\delta$, the weight vector $w$ computed by Convotron satisfies
If Assumptions \ref{assume} are satisfied then for $\eta = \Omega\left(\frac{\lambda_{\min}(P_\Sigma)}{k\lambda_{\max}(P)}\min \left(\frac{1}{\E_{x}[||x||^4]}, \frac{\epsilon ||w_*||^2}{\sqrt{\rho\E_{x}[||x||^4]}} \right)\right)$ and $T = O\left(\frac{k}{\eta\lambda_{\min}(P_\Sigma)} \log\left(\frac{1}{\epsilon \delta}\right)\right)$, with probability $1-\delta$, the weight vector $w$ computed by Convotron satisfies
\[
||w - w_*||^2 \leq \epsilon ||w_*||^2.
\]
\end{theorem}
\begin{proof}
Define $S_t = \{(x_1,y_1), \ldots, (x_t,y_t)\}$ The dynamics of Convotron can be expressed as follows:
\[
\E_{x_t, y_t}[||w_t - w_*||^2 - ||w_{t+1} - w_*||^2|S_{t-1}] = 2 \eta\E_{x_t, y_t}[(w_* - w_t)^T  G_t| S_{t-1}] - \eta^2 \E_{x_t, y_t}[||G_t||^2|S_{t-1}].
\]

We need to bound the RHS of the above equation. We have,
\begin{align}
\E_{x_t, y_t}[(w_* - w_t)^T  G_t | S_{t-1}] & = \E_{x_t, y_t}\left[(w_* - w_t)^T  (y_t - f_{w_t}(x_t))  \left(\sum_{i=1}^k P_i  x_t\right) \middle| S_{t-1}\right]\nonumber\\
& = \E_{x_t, \xi_t}\left[(w_* - w_t)^T  (f_{w_*}(x_t) + \xi_t  - f_{w_t}(x_t))  \left(\sum_{i=1}^k P_i  x_t\right) \middle| S_{t-1}\right]\nonumber\\
& = \E_{x_t}\left[(w_* - w_t)^T  (f_{w_*}(x_t) - f_{w_t}(x_t))  \left(\sum_{i=1}^k P_i  x_t\right) \middle| S_{t-1}\right] \label{eq:mean0}\\
&=\frac{1}{k}\sum_{1 \leq i,j \leq k} \E_{x_t}[(\sigma(w_*^T  P_i  x_t) - \sigma(w_t^T  P_i  x_t)) (w_*^T - w_t^T)  P_j  x_t | S_{t-1}]\nonumber\\
&= \frac{1 + \alpha}{2k}\sum_{1 \leq i,j \leq k}\E_{x_t}[((w_*^T - w_t^T)  P_i  x_t) ((w_*^T - w_t^T)  P_j  x_t) | S_{t-1}] \label{eq:useswitch}\\
&= \frac{1 + \alpha}{2k}(w_*^T - w_t^T)  \left(\sum_{1 \leq i \leq k} P_i\right)  \E_{x_t}[x_t x_t^T] \left(\sum_{1 \leq j \leq k} P_j^T\right)(w_* - w_t)\nonumber \\
& = \frac{1 + \alpha}{2k}(w_*^T - w_t^T)  \left(\sum_{1 \leq i, j \leq k} P_i \Sigma P_j^T\right) (w_* - w_t)\nonumber\\
& = \frac{1 + \alpha}{2k}(w_*^T - w_t^T)  P_{\Sigma} (w_* - w_t)\nonumber\\
& \geq \frac{1 + \alpha}{2k}\lambda_{\min}(P_\Sigma)||w_* - w_t||^2 \label{eq:lambda}.
\end{align}

(\ref{eq:mean0}) follows using linearity of expectation and the fact that that $\E[\xi_t|x_t] = 0$ and (\ref{eq:useswitch}) follows from using Lemma \ref{lem:switch}. (\ref{eq:lambda}) follows from observing that $P_{\Sigma}$ is symmetric, thus $\forall x, x^TP_{\Sigma}x \geq \lambda_{\min}(P_{\Sigma})||x||^2$. 

% Using the above, we get,
Now we bound the variance of $G_t$. Note that $\E[G_t] = 0$. Further,
\begin{align}
\E_{x_t, y_t}[||G_t||^2|S_{t-1}] &= \E_{x_t, y_t}\left[(y_t - f_{w_t}(x_t))^2\left|\left|\sum_{i=1}^k P_i  x_t\right|\right|^2 \middle| S_{t-1}\right]\nonumber\\
&\leq \lambda_{\max}(P) \E_{x_t, y_t}\left[(y_t - f_{w_t}(x_t))^2||x_t||^2 \middle| S_{t-1}\right]\label{eq:norm}\\
&= \lambda_{\max}(P)\E_{x_t, \xi_t}\left[(f_{w_*}(x_t) + \xi_t - f_{w_t}(x_t))^2||x_t||^2 \middle| S_{t-1}\right]\nonumber\\
&= \lambda_{\max}(P)\E_{x_t, \xi_t}\left[((f_{w_*}(x_t) - f_{w_t}(x_t))^2 + \xi_t^2 + 2 (f_{w_*}(x_t) - f_{w_t}(x_t)) \xi_t||x_t||^2 \middle| S_{t-1}\right]\nonumber\\
&= \lambda_{\max}(P) \left(\E_{x_t}\left[ (f_{w_*}(x_t) - f_{w_t}(x_t))^2||x_t||^2 \middle| S_{t-1}\right] + \E_{x_t, \xi_t}[\xi_t^2||x_t||^2] \right)\label{eq:mean00}\\
&\leq \lambda_{\max}(P) \left(\E_{x_t}[||x_t||^4] ||w_* - w_t||^2 + \sqrt{\rho\E_{x_t}[||x_t||^4]}\right) \label{eq:uselem}
\end{align}
(\ref{eq:norm}) follows from observing that $\left|\left|\sum_{i=1}^k P_i x\right|\right|^2 \leq \lambda_{\max}(P) ||x||^2$ for all $x$, (\ref{eq:mean00}) follows from observing that $\E_{\xi}[\xi|x] = 0$ and (\ref{eq:uselem}) follows from using Lemma $\ref{lem:sq}$ and bounding $\E_{x_t, \xi_t}[\xi_t^2||x_t||^2]$ using Cauchy-Schwartz inequality.

Combining the above equations and taking expectation over $S_{t-1}$, we get 
\[
\E_{S_t}[||w_{t+1} - w_*||^2] \leq (1 - 3\eta \beta + \eta^2 \gamma)\E_{S_{t-1}}[||w_t - w_*||^2] + \eta^2B
\]
for $\beta = \frac{(1 + \alpha)\lambda_{\min}(P_\Sigma)}{3k}$, $\gamma = \lambda_{\max}(P)\E_{x}[||x||^4]$ and $B = \lambda_{\max}(P)\sqrt{\rho\E_{x}[||x||^4]}$.

We set $\eta = \beta \min \left(\frac{1}{\gamma}, \frac{\epsilon||w_*||^2}{B} \right)$ and break the analysis to two cases:
\begin{itemize}
   \item \textbf{Case 1}: $\E_{S_{t-1}}[||w_t - w_*||^2] > \frac{\eta B}{\beta}$. This implies that $\E_{S_t}[||w_{t+1} - w_*||^2] \leq (1 - \eta \beta)\E_{S_{t-1}}[||w_t - w_*||^2]$.
   \item \textbf{Case 2}: $\E_{S_{t-1}}[||w_t - w_*||^2] \leq \frac{\eta B}{\beta} \leq \epsilon ||w_*||^2$.
\end{itemize}
Observe that once Case 2 is satisfied, we have $\E_{S_t}[||w_{t+1} - w_*||^2] \leq (1 - 2\eta \beta)\frac{\eta B}{\beta} + \eta^2B \leq \frac{\eta B}{\beta}$. Hence, for any iteration $>t$, Case 2 will continue to hold true. This implies that either at each iteration $\E_{S_{t-1}}[||w_t - w_*||^2]$ decreases by a factor $(1 - \eta \beta)$ or it is less than $\epsilon ||w_*||^2$. Thus if Case 1 is not satisfied for any iteration up to $T$, then we have,
\[
\E_{X_T}[||w_{T+1} - w||^2] \leq \left(1 - \eta \beta\right)^T ||w_*||^2 \leq e^{-\eta \beta T}||w_*||^2
\]
since at initialization $||w_1 - w_*|| = ||w_*||$. Setting $T = O\left(\frac{1}{\eta \beta} \log\left(\frac{1}{\epsilon \delta}\right)\right)$ and using Markov's inequality, with probability $1- \delta$, over the choice of $S_T$,
\[
||w_{T+1} - w_*|| \leq \epsilon ||w_*||^2.
\]
\end{proof}
By using Lemma \ref{lem:loss}, we can get a bound on $L(w_T) \leq \epsilon ||w_*||^2$ by appropriately scaling $\epsilon$.

\subsection{Convotron in the Realizable Case}
For the realizable (no noise) setting, that is, for all $(x,y) \sim \D$, $y = f_{w_*}(x)$, for some unknown $w_*$, Convotron achieves faster convergence rates.

\begin{corollary}\label{cor:realizable}
If Assumptions \ref{assume} are satisfied with the learning model restricted to the realizable case, then for suitably choosen $\eta$, after $T = O\left(\frac{k^2 \lambda_{\max}(P)\E_{x}[||x||^4]}{\lambda_{\min}(P_\Sigma)^2}\log\left(\frac{1}{\epsilon \delta}\right)\right)$ iterations, with probability $1-\delta$, the weight vector $w$ computed by Convotron satisfies
\[
||w - w_*||^2 \leq \epsilon ||w_*||^2.
\]
\end{corollary}
\begin{proof}
Since the setting has no noise, $\rho = 0$. Setting that parameter in Theorem \ref{thm:main} gives us $\eta = \Omega \left(\frac{\lambda_{\min}(P_\Sigma)}{k\lambda_{\max}(P)\E_{x}[||x||^4]}\right)$ as $\frac{\epsilon ||w_*||^2}{\sqrt{\rho\E_{x}[||x||^4]}}$ tends to infinity as $\rho$ tends to 0 and taking the minimum removes this dependence from $\eta$. Substituting this $\eta$ gives us the required result.
\end{proof}
Observe that the dependence of $\epsilon$ in the convergence rate is $\log(1/\epsilon)$ for the realizable setting, compared to the $1/\epsilon$ dependence in the noisy setting.
\section{Which Patch Structures are Easy to Learn?} \label{sec:patch}
In this section, we will show that the commonly used convolutional filters in practice (``patch and stride'') have good eigenvalues giving us fast convergence by Theorem \ref{thm:main}. We will start with the 1D case and then subsequently extend the result for the 2D case.
\subsection{1D Convolution}
Here we formally describe a patch and stride
convolution in the one-dimensional setting.  Consider a 1D image of dimension $n$. Let the patch size be $r$ and
stride be $d$. Let the patches be indexed from 1 and let patch $i$
start at position $(i - 1) d + 1$ and be contiguous through position $(i - 1) d + r$. The matrix $P_i$ of dimension $r \times n$ corresponding to patch $i$ looks as follows,
\[
P_i = \left(0_{r \times ((i - 1) d + 1)} I_{r} 0_{r \times (n - r - (i - 1) d) }\right)
\]
where $0_{a \times b}$ indicates a matrix of dimension $a \times b$ with all zeros and $I_{a}$ indicates the identity matrix of size $a$.

Thus, the total number of patches is $k = \lfloor \frac{n-r}{d} \rfloor + 1$. We will assume that $n \geq 2r - 1$ and $r \geq d$. The latter condition is to ensure there is some overlap, non-overlapping case, which is easier, is handled in the next section.

We will bound the extremal eigenvalues of $P = \sum_{i,j = 1}^{k} P_iP_j^T$. Simple algebra gives us the following structure for $P$,
\[
P_{i,j} = 	\begin{cases}
				k - a &\text{if } |i-j| = ad  \\
				0 &\text{otherwise}
			\end{cases}
\]

For understanding, we show the matrix structure for $d = 1$ and $n \geq 2 r$.
\[
\begin{pmatrix} 
    k   		& k - 1         & \dots     & k - r + 1 \\
    k - 1       & k     & \dots     & k - r + 2 \\
    \vdots      & \vdots        & \ddots    & \vdots  \\
    k - r + 1  & k - r + 2    & \dots     & k
    \end{pmatrix}.
\]

\subsubsection{Bounding Extremal Eigenvalues of $P$}
%Let $r = pd + q$ for some $p \geq 0$ and $1 \leq q \leq d$.  
The following lemmas bound the extremal eigenvalues of $P$.
\begin{lemma}\label{lem:max}
Maximum eigenvalue of $P$ satisfies $\lambda_{\max}(P) \leq k (p+1) - (p - p_2)(p_2 + 1) = O(kp)$ where $p = \lfloor\frac{r - 1}{d}\rfloor$ and $p_2 =  \lfloor\frac{p}{2}\rfloor$.
\end{lemma}
\begin{proof}
Using Theorem \ref{thm:circle}, we have $\lambda_{\max}(P) \leq \max_i \left(P_{i,i} + \sum_{j \neq i}|P_{i,j}|\right) = \max_i \sum_{j=1}^{k}P_{i,j}$. Observe that $P$ is bisymmetric thus $\sum_{j=1}^{k}P_{i,j} = \sum_{j=1}^{k}P_{r-i + 1,j}$ and we can restrict to the top half of the matrix. The structure of $P$ indicates that in a fixed row, the diagonal entry is maximum and the non-zero entries decrease monotonically by 1 as we move away from the diagonal. Also, there can be at most $p + 1$ non-zero entries in any row. Thus the sum is maximized when there are $p+1$ non-zero entries and the diagonal entry is the middle entry, that is at position $p_2d + 1$. By simple algebra,
\[
\lambda_{\max}(P) \leq \sum_{j=1}^{k}P_{p_2d + 1,j} = k + 2\sum_{j = 1}^{p_2} (k - j) + (p - 2p_2)(k - p_2 - 1) = k (p+1) - (p - p_2)(p_2 + 1).
\]
\end{proof}

\begin{lemma}\label{lem:min}
Minimum eigenvalue of $P$ satisfies $\lambda_{\min}(P) \geq 0.5$.
\end{lemma}
\begin{proof}
We break the analysis into following two cases:

\noindent\textbf{Case 1: $d < r/2$}\\
We can show that $\lambda_{\max}(P^{-1}) \geq 2$ using the structure of $P$ (see Lemma \ref{lem:min1} and \ref{lem:min2}). Since $\lambda_{\min}(P) = 1/\lambda_{\max}(P^{-1})$, we have $\lambda_{\min}(P) \geq 0.5$.

\noindent\textbf{Case 2: $d \geq r/2$}\\
In this case we directly bound the minimum eigenvalue of $P$. Using Theorem \ref{thm:circle}, we know that $\lambda_{\min}(P) \geq \min_i\left( P_{i,i} - \sum_{j \neq i}|P_{i,j}|\right)$. For $P_{i,j} \neq 0$, $|i - j| = ad$ for some $a$. The maximum value that $|i-j|$ can take is $r-1$ and since $d \geq r/2$, $a$ must be either 0 or 1. Also, for any $i$, there exists a unique $j$ such that $|i-j| = d$ since $r/2 \leq d < r$, thus there are exactly 2 non-zero entries in each row of $P$, $P_{i,i}$. This gives us, for each $i$, $\sum_{j \neq i} P_{i,j} = k - 1$. Thus, we get that $\lambda_{\min}(P) \geq \min_i \left(P_{i,i} - \left|\sum_{j \neq i} P_{i,j}\right|\right) = k - (k - 1) = 1$.

Combining both, we get the required result.
\end{proof}
% \begin{lemma}
% The maximum eigenvalue of $P^{-1}$, $\lambda_{\max}(P^{-1}) \leq 2$. Thus, $\lambda_{\min}(P) \geq 0.5$.
% \end{lemma}
% \begin{proof}
% By definition $P$ is positive semi-definite thus $P^{-1}$ is also positive semi-definite. We have,
% \begin{align*}
% &\lambda_{\max}(P^{-1}) = \max_{x \in S^{r-1}} \left(x^TP^{-1}x \right) \\
% & = \max_{x \in S^{r-1}}\left(\alpha x_1^2 + \sum_{i=2}^{r-1} x_i^2 + \alpha x_r^2 + 2 \beta x_1x_r + \sum_{j = 1}^{r-1} x_jx_{j+1}\right)\\
% & \leq \max_{x \in S^{r-1}}\left( \sum_{i=1}^{r} x_i^2 + |x_1x_r| + \sum_{j = 1}^{r-1} |x_jx_{j+1}|\right)\\
% & \leq \max_{x \in S^{r-1}} \left(\sum_{i=1}^{r} x_i^2 + \frac{x_1^2 + x_r^2}{2} + \sum_{j = 1}^{r-1} \frac{x_j^2 + x_{j+1}^2}{2}\right)\\
% & = \max_{x \in S^{r-1}}\left( \sum_{i=1}^{r} 2x_i^2\right) = 2.
% \end{align*}
% The first inequality follows from  $\beta \leq 0.5$, $\alpha \leq 1$ and $a \leq |a|$ for any $a$. The second inequality follows from the AM-GM inequality, $|ab| \leq \frac{a^2 + b^2}{2}$. Lastly, we use the fact that $||x|| = 1$.

% The second part of the lemma follows from the property $\lambda_{\max}(P^{-1}) = 1/ \lambda_{\min}(P)$.
% \end{proof}
\subsubsection{Learning Result for 1D}
Augmenting the above analysis with Theorem \ref{thm:main} gives us learnability of 1D convolution filters.
\begin{corollary}\label{cor:1D}
\sloppy
If Assumptions 1(a),(b), and (d) are satisfied and the patches have a patch and stride structure with parameters $n, r, d$, then for suitably chosen $\eta$ and $T = O\left(\frac{n^3r}{d^4\lambda_{\min}(\Sigma)^2}\max \left(\E_{x}[||x||^4], \frac{\sqrt{\rho\E_{x}[||x||^4]}}{\epsilon ||w_*||^2} \right) \log\left(\frac{1}{\epsilon \delta}\right)\right)$, with probability $1-\delta$, the weight vector $w$ output by Convotron satisfies
\[
||w - w_*||^2 \leq \epsilon ||w_*||^2.
\]
\end{corollary}
\begin{proof}
Combining the above Lemmas gives us that $\lambda_{\max}(P) = O(pk) = O(nr/d^2)$ and $\lambda_{\min}(P) = \Omega(1)$. Observe that $\lambda_{\min}(P_\Sigma) \geq \lambda_{\min}(P)\lambda_{\min}(\Sigma)$. Substituting these values in Theorem \ref{thm:main} gives us the desired result.
\end{proof}
Comparing with SGD, \cite{brutzkus2017globally} showed that even for $r = 2$ and $d=1$, Gradient descent can get stuck in a local minima with probability $\geq 1/4$.

\subsection{2D Convolution}
Here we formally define stride and patch convolutions in two
dimensions.  Consider a 2D image of dimension $n_1 \times n_2$. Let the patch size be $r_1 \times r_2$ and stride in both directions be $d_1, d_2$ respectively. Enumerate patches such that patch $(i, j)$ starts at position $((i-1)  d_1 + 1, (j-1)  d_2 + 1)$ and is a rectangle with diagonally opposite point $((i-1)  d_2 + r_1, (j-1)  d_2 + r_2)$. Let $k_1 = \lfloor \frac{n_1-r_1}{d_1} \rfloor + 1$ and $k_2 = \lfloor \frac{n_2-r_2}{d_2} \rfloor + 1$. Let us vectorize the image row-wise into a $n_1  n_2$ dimension vector and enumerate each patch row-wise to get a $r_1  r_2$ dimensional vector.
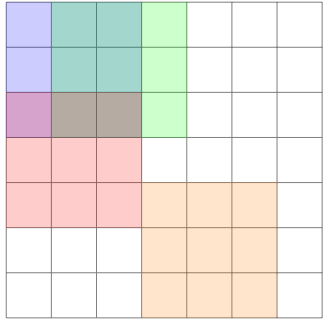
\begin{figure}
\centering
\begin{tikzpicture}[scale=0.6]
\draw[step=1cm,gray,very thin] (0,0) grid (7,7);
\fill[color=blue, opacity=0.2] (0,7) rectangle (3,4);
\fill[color=green, opacity=0.2] (1,7) rectangle (4,4);
\fill[color=red, opacity=0.2] (0,5) rectangle (3,2);
\fill[color=orange, opacity=0.2] (3,3) rectangle (6,0);
\end{tikzpicture}
\caption{2D convolution patches for image size $n_1 = n_2 = 7$, patch size $r_1 = r_2 = 3$, and stride $d_1 = 2$, $d_2 = 1$. Blue box corresponds to patch $(1,1)$, red to patch $(2, 1)$ green to patch $(1,2)$ and orange to patch $(3, 4)$.}
\end{figure}
Let $Q_{(i,j)}$ be the indicator matrix of dimension $r_1  r_2 \times n_1  n_2$ with 1 at $(a,b)$ if the $a$th location of patch $(i,j)$ is $b$. More formally, $(Q_{(i,j)})_{a,b} = 1$ for all $a = p  r_2 + q + 1$ for $0 \leq p < r_1$, $0 \leq q < r_2$, and $b = ((i-1)  d_1 + p)  n_2 + j  d_2 + q + 1$ else 0. Note that there are $k_1 \cdot k_2$ patches in total with the corresponding patch matrices being $Q_{(i,j)}$ for $1 \leq i \leq k_1, 1 \leq j \leq k_2$.

\subsubsection{Bounding Extremal Eigenvalues of $Q$}

We will bound the extremal eigenvalues of $Q = \sum_{i,p = 1}^{k_1 }\sum_{j,q = 1}^{k_2} Q_{(i,j)} Q_{(p,q)}^T$. Let $P^{(1)}_i$'s be the patch matrices corresponding to the 1D convolution for parameters $n_1,r_1,d_1$ defined as in the previous section and let $P^{(1)} = \sum_{i,j = 1}^{k_1} P^{(1)}_i (P^{(1)}_j)^T$. Define $P^{(2)}_i$'s for $1 \leq i \leq k_2$ and $P^{(2)}$ similarly with parameters $n_2, r_2, d_2$ instead of $n_1,r_1,d_1$.%$P^{(2)} = \sum_{i,j = 1}^{k_2} P^{(1)}_i (P^{(1)}_j)^T$. 

\begin{lemma}
$Q_{(i,j)} = P^{(1)}_i \otimes P^{(2)}_j$.
\end{lemma}
\begin{proof}
Intuitively $P^{(1)}_i$ and $P^{(2)}_j$ give the indices corresponding to the row and column of the 2D patch and the Kronecker product vectorizes it to give us the $(i,j)$th patch. More formally, we will show that $(Q_{(i,j)})_{a,b} = 1$ iff $ (P^{(1)}_i \otimes P^{(2)}_j)_{a,b} = 1$.

Let $a = p  r_2 + q + 1$ with $0 \leq p < r_1$, $0 \leq q < r_2$ and $b = r  n_2 + s + 1$ with $0 \leq r < n_1$, $0 \leq s < n_2$. Then, $(P^{(1)}_i \otimes P^{(2)}_j)_{a,b} = 1$ iff $(P^{(1)}_i)_{p,r} = 1$ and $(P^{(2)}_j)_{q,s} = 1$. We know that $(P^{(1)}_i)_{p,r} = 1$ iff $r = (i-1)  d_1 + p + 1$ and $(P^{(2)}_j)_{q,s} = 1$ iff $s = (j - 1)  d_2 + q + 1$. This gives us that $b = ((i - 1)  d_1 + p)  n_1 +(j -1)  d_2 + q + 1$, which is the same condition for $(Q_{(i,j)})_{a,b} = 1$. Thus $Q_{(i,j)} = P^{(1)}_i \otimes P^{(2)}_j$.
\end{proof}

\begin{lemma}
$Q = P^{(1)} \otimes P^{(2)}$.
\end{lemma}
\begin{proof}
We have,
\begin{align*}
Q &= \sum_{i,p = 1}^{k_1}\sum_{j,q = 1}^{k_2} Q_{(i,j)} Q_{(p,q)}^T\\
& = \sum_{i,p = 1}^{k_1}\sum_{j,q = 1}^{k_2} (P^{(1)}_i \otimes P^{(2)}_j) (P^{(1)}_p \otimes P^{(2)}_q)^T\\
& = \sum_{i,p = 1}^{k_1}\sum_{j,q = 1}^{k_2} (P^{(1)}_i \otimes P^{(2)}_j) ((P^{(1)}_p)^T \otimes (P^{(2)}_q)^T)\\
& = \sum_{i,p = 1}^{k_1}\sum_{j,q = 1}^{k_2}  (P^{(1)}_i (P^{(1)}_p)^T) \otimes (P^{(2)}_j (P^{(2)}_q)^T)\\
& = \left(\sum_{i,p = 1}^{k_1} P^{(1)}_i (P^{(1)}_p)^T \right) \otimes \left(\sum_{j,q = 1}^{k_2} P^{(2)}_j (P^{(2)}_q)^T\right)\\
& = P^{(1)} \otimes P^{(2)}.
\end{align*}
\end{proof}

\begin{lemma}\label{lem:2D}
%For $d_1 = d_2 = 1$, $n_1 \geq 2 r_1$ and $n_2 \geq 2 r_2$
We have $\lambda_{\min}(Q) \geq 0.25$ and $\lambda_{\max}(Q) = O(k_1p_1k_2p_2)$ where $p_1 = \lfloor\frac{r_1 - 1}{d_1}\rfloor$ and $p_2 = \lfloor\frac{r_2 - 1}{d_2}\rfloor$.
\end{lemma}
\begin{proof}
Since $Q = P^{(1)} \otimes P^{(2)}$ and $Q, P^{(1)}, P^{(2)}$ are positive semi-definite, $\lambda_{\min}(Q) = \lambda_{\min}(P)\lambda_{\min}(P^{(2)})$ and $\lambda_{\max}(Q) = \lambda_{\max}(P^{(1)})\lambda_{\max}(P^{(2)})$. Using the lemmas from the previous section gives us the required result.
\end{proof}

Note that this technique can be extended to higher dimensional patch structures as well.

\subsubsection{Learning Result for 2D}
Similar to the 1D case, combining the above analysis with Theorem \ref{thm:main} gives us learnability of 2D convolution filters.
\begin{corollary}\label{cor:2D}
\sloppy
If Assumptions 1(a),(b), and (d) are satisfied and the patches have a 2D patch and stride structure with parameters $n_1, n_2, r_1, r_2, d_1, d_2$, then for suitably chosen $\eta$ and $T = O\left(\frac{n_1^3n_2^3r_1r_2}{d_1^3d_2^3\lambda_{\min}(\Sigma)^2} \max \left(\E_{x}[||x||^4], \frac{\sqrt{\rho\E_{x}[||x||^4]}}{\epsilon ||w_*||^2} \right) \log\left(\frac{1}{\epsilon \delta}\right)\right)$, with probability $1-\delta$, the weight vector $w$ output by Convotron satisfies
\[
||w - w_*||^2 \leq \epsilon ||w_*||^2.
\]
\end{corollary}
\begin{proof}
Lemma \ref{lem:2D} gives us that $\lambda_{\max}(Q) = O(n_1n_2r_1r_2/(d_1d_2)^2)$ and $\lambda_{\min}(P) = \Omega(1)$. Observe that $\lambda_{\min}(P_\Sigma) \geq \lambda_{\min}(P)\lambda_{\min}(\Sigma)$. Substituting these values in Theorem \ref{thm:main} gives us the desired result.
\end{proof}
\section{Non-overlapping Patches are Easy}
In this section, we will show that if there is one patch that does not overlap with any patch and the covariance matrix is identity then we can easily learn the filter even if the other patches have arbitrary overlaps. This includes the commonly used Gaussian assumption. WLOG we assume that $P_1$ is the patch that does not overlap with any other patch implying $P_1P_j^T = P_j^TP_1 = 0$ for all $j \neq 1$.

\begin{algorithm}[tb]
   \caption{Convotron-No-Overlap}
   \label{alg:convotronno}
\begin{algorithmic}
   \STATE Initialize $w_1 := 0 \in \mathbb{R}^r$.
   \FOR{$t=1$ {\bfseries to} $T$}
   \STATE Draw $(x_t, y_t) \sim \mathcal{D}$
   \STATE Let $G_t = (y_t - f_{w_t}(x_t))  P_1  x_t$
   \STATE Set $w_{t+1} = w_t + \eta G_t$
   \ENDFOR
   \STATE {Return $w_{T+1}$}
\end{algorithmic}
\end{algorithm}

Observe that the algorithm ignores the directions of all other patches
and yet succeeds. This indicates that with respect to a Gaussian
distribution, in order to have an interesting patch structure (for one
layer networks), it is necessary to avoid having even a single disjoint
patch. The following theorem shows the convergence of Convotron-No-Overlap.
\begin{theorem}\label{thm:main-no-overlap}
If Assumptions \ref{assume} are satisfied with $\Sigma = I$, then for $\eta = \frac{(1 + \alpha)}{3k}\min \left(\frac{1}{\E_{x}[||x||^4]}, \frac{\epsilon ||w_*||^2}{\sqrt{\rho\E_{x}[||x||^4]}} \right)$ and $T \geq \frac{1}{\eta \delta} \log\left(\frac{1}{\epsilon \delta}\right)$, with probability $1-\delta$, the weight vector $w$ outputted by Convotron-No-Overlap satisfies
\[
||w - w_*||^2 \leq \epsilon ||w_*||^2.
\]
\end{theorem}
\begin{proof}
The proof follows the outline of the Convotron proof very closely. We use the same definitions as in the previous proof. We have,
\begin{align*}
\E_{x_t, y_t}[(w_* - w_t)^T  G_t | S_{t-1}] &=\frac{1}{k}\sum_{1 \leq i \leq k} \E_{x_t}[(\sigma(w_*^T  P_i  x_t) - \sigma(w_t^T  P_i  x_t))(w_*^T- w_t^T)  P_1  x_t| S_{t-1}]\\
&= \frac{1 + \alpha}{2k}\sum_{1 \leq i \leq k}\E_{x_t}[((w_*^T - w_t^T)  P_i  x_t)((w_*^T- w_t^T)  P_1  x_t)| S_{t-1}] \\
&= \frac{1 + \alpha}{2k}(w_*^T - w_t^T)  \left(\sum_{1 \leq i \leq k} P_i\right)  \E_{x_t}[x_t x_t^T]  P_1(w_* - w_t) \\
&= \frac{1 + \alpha}{2k}||w_*^T - w_t^T||^2
\end{align*}
The last equality follows since $P_i^TP_1 = 0$ for all $ i \neq 1$ and $P_1^TP_1$ is a permutation of identity.

Similarly,
\begin{align*}
\E_{x_t, y_t}[||G_t||^2| S_{t-1}] &= \E_{x_t, y_t}\left[(y_t - f_{w_t}(x_t))^2\left|\left| P_i  x_t\right|\right|^2 \middle| S_{t-1}\right]\nonumber\\
&\leq \E_{x_t, y_t}\left[(y_t - f_{w_t}(x_t))^2||x_t||^2 \middle| S_{t-1}\right]\\
&\leq \E_{x_t}[||x_t||^4] ||w_* - w_t||^2 + \sqrt{\rho\E_{x_t}[||x_t||^4]}
\end{align*}

Following the rest of the analysis for $\eta$ and $T$ as in the theorem statement gives us the required result.
\end{proof}
\section{Experiments: SGD vs Convotron}
\begin{algorithm}[tb]
   \caption{SGD}
   \label{alg:sgd}
\begin{algorithmic}
   \STATE Randomly initialize $w_1 \in \mathbb{R}^r$.
   \FOR{$t=1$ {\bfseries to} $T$}
   \STATE Draw $(x_t, y_t) \sim \mathcal{D}$
   \STATE Let $G_t = (y_t - f_{w_t}(x_t))  \left(\sum_{i=1}^k \sigma^\prime(w_t^TP_i x_t) P_i  x_t\right)$
   \STATE Set $w_{t+1} = w_t + \eta G_t$
   \ENDFOR
   \STATE {Return $w_{T+1}$}
\end{algorithmic}
\end{algorithm}
To further support our theoretical findings, we empirically compare
the performance of SGD (Algorithm \ref{alg:sgd}) with our algorithm
Convotron. We measure performance based on the failure probability, that is, the fraction of runs the
algorithm fails to converge on randomly initialized runs (the randomness is
over both the choice of initialization for SGD and the draws from the distribution). More formally, we say that the algorithm fails if the closeness in $l_2$-norm of the difference of the final weight vector obtained $(w_T)$ and the true weight parameter ($w_*$), that is, $||w_T - w_*||$ is greater than a threshold $\theta$. 
We choose this
measure because in practice, due to the high computation
time of training neural networks, random restarts are
expensive.

%/unusual.  Thus, it is important to design algorithms that
%give initialization independent convergence guarantees.

In the experiments, given a fixed true weight vector, for varying
learning rates (increments of $0.01$), we choose 50 random
initializations and run the two algorithms with them as starting
points. We plot the failure probability ($\theta = 0.1$) with varying
learning rate. Note that the lowest learning rate we use is $0.01$ as
making the learning rate too small requires high number of iterations
for convergence for both algorithms.   

%Also, note that the increase in
%failure probability with a large learning rate is bound to happen for both algorithms due to very large updates.
\begin{figure}
\label{exp}
  \centering
  \begin{minipage}[b]{0.47\columnwidth}
    \includegraphics[width=\textwidth]{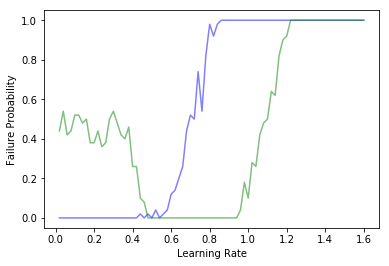}
  \end{minipage}
  \hfill
  \begin{minipage}[b]{0.47\columnwidth}
    \includegraphics[width=\textwidth]{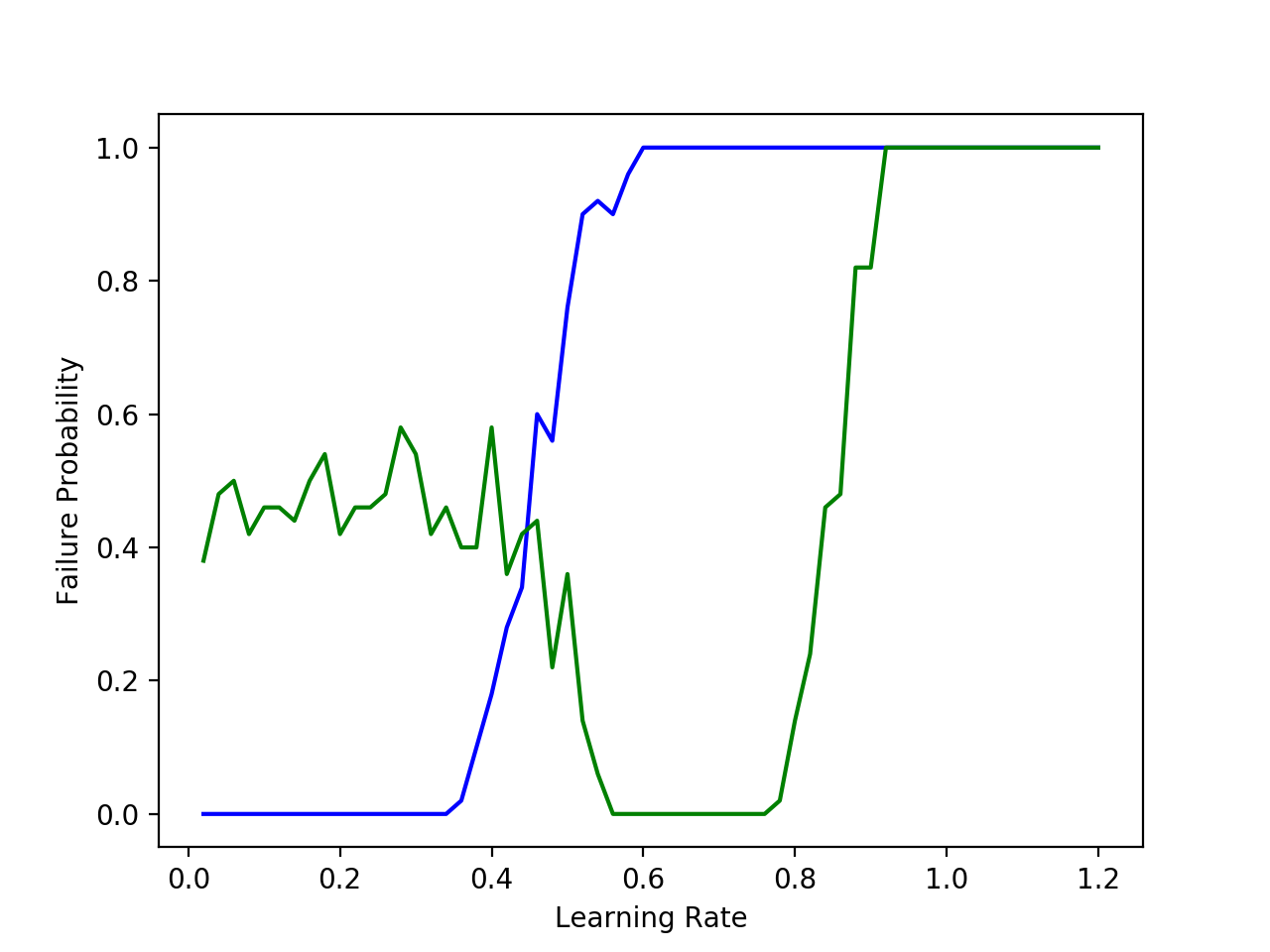}
  \end{minipage}
  \begin{minipage}[b]{0.47\columnwidth}
    \includegraphics[width=\textwidth]{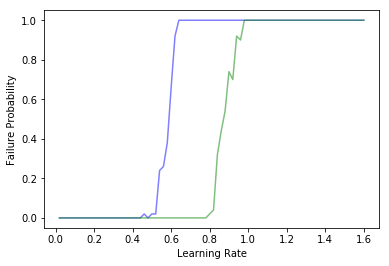}
  \end{minipage}
  \hfill
  \begin{minipage}[b]{0.47\columnwidth}
    \includegraphics[width=\textwidth]{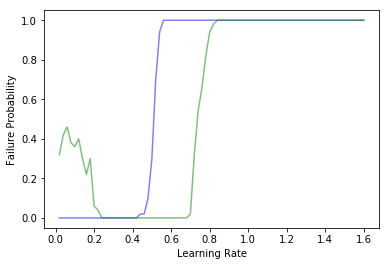}
  \end{minipage}
  \caption{Failure probability of SGD (green) vs Convotron (blue) with varying learning rate $\eta$. Experiment 1: Patch and stride 1D (Top-left) and  2D (Top-right). Experiment 2: Input distribution has mean 0 and covariance matrix identity (Bottom-left) and non-identity covariance matrix (Bottom-right). The curves are shifted due to scaling difference of updates.}
\end{figure}

We first test the performance on a simple 1D convolution case with
$(n,k,d, T) = (8, 4, 1, 6000)$ and 2D case with $(n_1, n_2, k_1, k_2,
d_1, d_2, T) = (5,5,3,3,1,1,15000)$ on inputs drawn from a normalized
($l_2$ norm 1) Gaussian distribution with identity covariance
matrix. We adversarially choose a fixed weight vector\footnote{We take
  the vector to be $[1, -1, 1, -1]$ in the 1D case and normalize. This
  weight vector can be viewed as an edge detection filter, that is,
  counting the number of times image goes from black (negative) to
  white (positive).} ($l_2$-norm 1). Figure \ref{exp} (Top) shows that SGD
has a small data dependent range where it succeeds but may fail with
almost $0.5$ probability outside this region whereas Convotron always
returns a good solution for small enough $\eta$ chosen according to
Theorem $\ref{thm:main}$. The failure points observed for SGD show the
prevalence of bad local minima where SGD gets stuck.

For the second experiment, we choose a fixed weight vector for which
SGD performs well with very high probability on a normalized Gaussian input
distribution with identity covariance matrix (see Figure \ref{exp} (Bottom-left)). However, on choosing a different covariance matrix with
higher condition number $\sim 60$, the performance of SGD worsens
whereas Convotron always succeeds (see Figure \ref{exp} (Bottom-Right)). The covariance matrix is generated by choosing random matrices followed by symmetrizing them and adding $c I$ for $c>0$ to make the eigenvalues positive.

These experiments demonstrate that techniques for fine-tuning SGD's
learning rate are necessary, even for very simple architectures.  In
contrast, no fine-tuning is necessary for Convotron: the correct
learning rate can be easily computed given the learner's desired patch
structure and estimate of the covariance martix.

%existence of weight vectors and input distributions for which SGD does not guarantee convergence without fine tuning the learning rate whereas Convotron does not require any hyperparameter tuning ($\eta$ chosen based on Theorem \ref{thm:main}).

\section*{Acknowledgments}
We thank Jessica Hoffmann and Philipp Kr\"{a}henb\"{u}hl for useful discussions.

% \begin{figure}[!tbp]
% \label{exp:2}
%   \centering
%   \begin{minipage}[b]{0.47\columnwidth}
%     \includegraphics[width=\textwidth]{files/exp2}
%   \end{minipage}
%   \hfill
%   \begin{minipage}[b]{0.47\columnwidth}
%     \includegraphics[width=\textwidth]{files/exp3}
%   \end{minipage}
%   \caption{Failure probability of SGD (green) vs Convotron (blue) with varying learning rate $\eta$. Input distribution has mean 0 and covariance matrix identity (left) and non-identity covariance matrix (right). The curves are shifted due to scaling difference of updates.}
% \end{figure}

% \input{files/conclusion}
\bibliography{full}
\bibliographystyle{alpha}
\appendix
\section{Omitted Proofs}
\subsection{Proof of Lemma 1}
We will follow the following notation:
\begin{align*}
T_1 &= \E_{x \sim \D}[\sigma(a^T  x)(b^T  x)]\\
T_2 &= \E_{x \sim \D}[(a^T  x) (b^T  x)].
\end{align*}
Since $x$ is drawn from a symmetric distribution we have $E_{x \sim \D}[F(x)] = E_{x \sim \D}[F(-x)]$ for any function $F$. Thus, we have
\begin{align*}
T_1 &= \E_{x \sim \D}[\sigma(-a^T  x)(-b^T  x)]\\
\implies 2 T_1 &= \E_{x \sim \D}[(\sigma(-a^T  x) - \sigma(-a^T  x))(b^T  x)]
\end{align*}
Observe that $\sigma(c) - \sigma(-c) = \frac{(1 - \alpha)|c| + (1 + \alpha)c}{2} - \frac{(1 - \alpha)|a| - (1 + \alpha)c}{2} = (1 + \alpha)c$. Substituting this in the above, we get the required result $2T_1  = (1 + \alpha)T_2$.

\subsection{Proof of Lemma 2}
We have,
\begin{align*}
\frac{1}{k}\sum_{1 \leq i \leq k} &\E_{x}[(\sigma(w_*  P_i  x) - \sigma(w  P_i  x))(w_* - w)^T  P_i  x)]\\
&= \frac{1 + \alpha}{2k}\sum_{1 \leq i \leq k} \E_{x}[((w_* - w)^T  P_i  x))^2]\\
&= \frac{1 + \alpha}{2k}(w_* - w_t)^T\left(\sum_{1 \leq i \leq k} P_i \E_x[xx^T] P_i^T \right)(w_*- w_t)\\
&= \frac{1 + \alpha}{2k}(w_* - w_t)^T\left(\sum_{1 \leq i \leq k} P_i \Sigma P_i^T \right)(w_* - w_t)\\
&\leq \frac{(1 + \alpha)\lambda_{\max}(\Sigma)}{2k}\left(\sum_{1 \leq i \leq k} \lambda_{\max}(P_i P_i^T) \right)||w_* - w_t||^2\\
& = \frac{(1 + \alpha)\lambda_{\max}(\Sigma)}{2} ||w_* - w||^2
\end{align*}
The first equality follows from using Lemma \ref{lem:switch} and the last follows since for all $i$, $P_iP_i^T$ is a permutation of the identity matrix by definition.

Using monotonicity of $\sigma$ and Jensen's inequality, we also have,
\begin{align*}
\frac{1}{k}\sum_{1 \leq i \leq k} &\E_{x}[(\sigma(w_*  P_i  x) - \sigma(w  P_i  x))(w_* - w)^T  P_i  x]\\
& \geq \frac{1}{k}\sum_{1 \leq i \leq k} \E_{x}[(\sigma(w_*  P_i  x) - \sigma(w  P_i  x))^2]\\
& \geq  \E_{x}\left[\left(\frac{1}{k}\sum_{1 \leq i \leq k}(\sigma(w_*  P_i  x) - \sigma(w  P_i  x))\right)^2\right]\\
& = L(w).
\end{align*}
Combining the two above lemmas, we get the required result.

\subsection{Proof of Lemma 3}
We have,
\begin{align*}
(f_{w_*}(x) - f_{w_t}(x))^2 &= \left(\frac{1}{k}\sum_{i=1}^k (\sigma(w_*^T  P_i  x) - \sigma(w^T  P_i  x)) \right)^2\\
&\leq \frac{1}{k}\sum_{i=1}^k (\sigma(w_*^T  P_i  x) - \sigma(w^T  P_i  x))^2\\
&\leq \frac{1}{k}\sum_{i=1}^k (w_*^T  P_i  x - w^T  P_i  x)^2\\
&\leq \frac{1}{k}\sum_{i=1}^k ||w_* - w||^2\lambda_{\max}(P_iP_i^T)||||x||^2\\
&\leq ||w_* - w||^2||x||^2 \label{eq:finbound}
\end{align*}
The first inequality follows from using Jensen's, the second inequality follows from the 1-Lipschitz property of $\sigma$, the third follows from observing that $P_iP_i^T$ is a PSD matrix and the last inequality follows since for all $i$, $\lambda_{\max}(P_iP_i^T) = 1$ since $P_i P_i^T$ is a permutation of the identity matrix.
\section{Properties of Patch Matrix $P$}
Let $r = pd + q$ for some $p \geq 0$ and $1 \leq q \leq d$.
\begin{lemma}\label{lem:min1}
For $d < r/3$, $P^{-1}$ has the following form:
\[
P^{-1}_{i,j} = 	\begin{cases}
				\alpha_0 \text{ if } i = j \in \{1, \ldots, q\} \cup \{r-q + 1, \ldots ,r\}\\
				\alpha_1 \text{ if } i = j \in \{q+1, \ldots, d\} \cup \{r - d + 1, r - q \}   \\
				1 \text{ if } i = j \in \{d + 1, \ldots, r-d\} \\
				-0.5 \text{ if } |i - j| = d + 1\\
				% -0.5 \text{ if } |i - j| = d \text{ and } i \text{ or } j \in \{d, \ldots, r-d - 1\}\\
				% \gamma \text{ if } |i - j| = d \text{ and } i, j \not\in \{d, \ldots, r- d - 1\} \\
				\phi \text{ if } |i - j| = (p - 1)d + 1 \text{ and } i \text{ or } j \in \{q+1, \ldots, d\}\\
				\beta  \text{ if } |i - j| = pd + 1 \\
				0 \text{ otherwise}
			\end{cases}
\]
where $\alpha_0 = \beta + 0.5$, $\alpha_1 = \phi + 0.5$, $\beta = \frac{0.5}{2k - p}$ and $\phi = \frac{0.5}{2k - p - 1}$. Also, $\lambda_{\max}(P^{-1}) \leq 2$.
\end{lemma}
\begin{proof}
We need to show that $ A = P P^{-1} = I$. Observe that $P$ and $P^{-1}$ are bisymmetric, thus $A$ is centrosymmetric implying $A_{i,j} = A_{r - 1 - i, r - 1 - j}$. Hence, we need to only prove that the lower triangular matrix matches $I$. We show the result for $p >2$, as the same ideas apply for the other case.

To verify this, consider each diagonal entry, 
\begin{itemize}
	\item $d \leq i \leq \lceil d/2 \rceil$: $A_{i,i} =  -0.5 (k - 1) + k - 0.5(k - 1) = 1$. 
	\item $i \in \{1, \ldots, q\}$: $A_{i,i} = \alpha_0 k - 0.5 (k - 1) + \beta \left(k - p\right) = 1$.
	\item $i \in \{q+1, \ldots, d\}$: $A_{i,i} = \alpha_1 k - 0.5 (k - 1) + \phi \left(k - p - 1\right) = 1$.
\end{itemize}

For non-diagonal entries, that is, $j \neq i$, 
\begin{itemize}
	\item $d \leq j \leq \lceil d/2 \rceil$: $A_{i,j} = -0.5 P_{i,j-d} + P_{i,j} -0.5 P_{i, j+d}$. If $|i - j| = a d$ then $A_{i,j} = -0.5\left(k  - a - 1\right) + k - a - 0.5\left(k  - a + 1\right) = 0$, else $P_{i,j} = P_{i,j-d} = P_{i,j+d} = 0 \implies A_{i,j} = 0$.
	\item $j \in \{1, \ldots, q\}$: $A_{i,j} = \alpha_0 P_{i,j} - 0.5 P_{i,j + d} + \beta P_{i, j + pd}$. Now if $i - j = a d$, then $A_{i,j} = \alpha_0(k - a) - 0.5 (k - a + 1) + \beta (k - p + a) = 0$ else $P_{i,j} = P_{i,j + d} =P_{i, j + pd} = 0 \implies A_{i,j} = 0$.
	\item $j \in \{q+1, \ldots, d\}$: $A_{i,j} = \alpha_1 P_{i,j} - 0.5 P_{i,j + d} + \beta P_{i, j + pd}$. Now if $i - j = a d$, then $A_{i,j} = \alpha_1(k - a) - 0.5 (k - a + 1) + \phi (k - p + a + 1) = 0$ else $P_{i,j} = P_{i,j + d} = P_{i, j + pd} = 0 \implies A_{i,j} = 0$.
\end{itemize}
Hence $A = I$.

Using Theorem \ref{thm:circle}, we have $\lambda_{\max}(P^{-1}) = \max_i \left(P^{-1}_{i,i} + \sum_{j \neq i} |P^{-1}_{i,j}|\right)$. If $q < d$, then $\lambda_{\max}(P^{-1}) =  \max( \alpha_0 + 0.5 + \beta, \alpha_1 + 0.5 + \phi, 1 + 0.5 + 0.5) = \max(2 \beta + 1, 2 \phi + 1, 2) = 2$ as $\beta, \phi \leq 0.5$ which follows from $2k - p - 1 \geq 1$. Similarly, when $q = d$, $\lambda_{\max}(P^{-1}) =  \max( \alpha_0 + 0.5 + \beta, 1 + 0.5 +0.5) = \max(2\beta + 1,2) = 2$.
\end{proof}

\begin{figure}
\begin{center}
\begin{tikzpicture}
  \matrix[matrix of math nodes,left delimiter = (,right delimiter = ),row sep=2pt,column sep = 2pt,nodes={minimum size=0.5em}] (m)
  {
    \alpha  & -0.5      & ~         &~     & ~           & \beta \\
    -0.5    & 1         & -0.5      & ~          & ~    & ~ \\
     ~       & \ddots    & \ddots    &\ddots     &   ~          & ~\\
      ~      & ~  & \ddots    &\ddots     &\ddots      & ~\\
       ~     &  ~         &   ~        &-0.5       & 1           & -0.5 \\
    \beta   &  ~        &  ~    &~          &-0.5         & \alpha\\
  };
  \begin{pgfonlayer}{background}
    \node[inner sep=3pt,fit=(m-3-1)]          (1)   {};
    \node[inner sep=3pt,fit=(m-5-1)]          (2)   {};
    \node[inner sep=3pt,fit=(m-6-2) (m-6-3)(m-6-4)]          (3)   {};
    \node[inner sep=3pt,fit=(m-4-6)]          (4)   {};
    \node[inner sep=3pt,fit=(m-2-6)]          (5)   {};
    \node[inner sep=3pt,fit=(m-1-3) (m-1-4)(m-1-5)]          (6)   {};
    \draw[dotted,fill=green!50!white,inner sep=3pt,fill opacity=0.1]
(1.north west) -- (3.south east) -- (3.south west) -| (2.south east) -| (2.south west) -- (1.north west) -- cycle;
\draw[dotted,fill=green!50!white,inner sep=3pt,fill opacity=0.1]
(4.south east) -- (6.north west) -- (6.north east) -| (5.north west) -| (5.north east)-- (4.south east) -- cycle;
  \end{pgfonlayer}
\end{tikzpicture}
\end{center}
\caption{$P^{-1}$ for $d = 1$. Here $\alpha = \beta + 0.5$ and $\beta = \frac{0.5}{2k - p} = \frac{0.5}{2n - 3r + 3}$. The shaded area is all 0s.}
\end{figure}

\begin{lemma}\label{lem:min2}
For $r/3 \leq d < r/2$, $P^{-1}$ has the following form:
\[
P^{-1}_{i,j} = 	\begin{cases}
				\alpha_0 \text{ if } i = j \in \{1, \ldots, q\} \cup \{r-q+1, \ldots ,r\}\\
				\alpha_1 \text{ if } i = j \in \{q+1, \ldots, d\} \cup \{r - d+1, r - q \}   \\
				1 \text{ if } i = j \in \{d+1, \ldots, r-d\} \\
				-0.5 \text{ if } |i - j| = d+1 \text{ and } i \text{ or } j \in \{d+1, \ldots, r - d\}\\
				\phi \text{ if } |i - j| = d+1 \text{ and } i \text{ or } j \in \{q+1, \ldots, d\}\\
				\beta  \text{ if } |i - j| = 2d + 1 \\
				0 \text{ otherwise}
			\end{cases}
\]
where $\alpha_0 = \beta + 0.5$, $\alpha_1 = \frac{k}{2k - 1}$, $\beta = \frac{0.5}{2k - 2}$ and $\phi = -\frac{k - 1}{2k - 1}$. Also, $\lambda_{\max}(P^{-1}) \leq 2$.
\end{lemma}
\begin{proof}
Similar to the previous lemma, to verify this, consider each diagonal entry, 
\begin{itemize}
	\item $d \leq i \leq \lceil d/2 \rceil$: $A_{i,i} =  -0.5 (k - 1) + k - 0.5(k - 1) = 1$. 
	\item $i \in \{1, \ldots, q\}$: $A_{i,i} = \alpha_0 k - 0.5 (k - 1) + \beta \left(k - p\right) = 1$.
	\item $i \in \{q+1, \ldots, d\}$: $A_{i,i} = \alpha_1 k + \phi (k - 1) = 1$.
\end{itemize}

For non-diagonal entries, that is, $j \neq i$, 
\begin{itemize}
	\item $d \leq j \leq \lceil d/2 \rceil$: $A_{i,j} = -0.5 P_{i,j-d} + P_{i,j} -0.5 P_{i, j+d}$. If $|i - j| = a d$ then $A_{i,j} = -0.5\left(k  - a - 1\right) + k - a - 0.5\left(k  - a + 1\right) = 0$, else $P_{i,j} = P_{i,j-d} = P_{i,j+d} = 0 \implies A_{i,j} = 0$.
	\item $j \in \{1, \ldots, q\}$: $A_{i,j} = \alpha_0 P_{i,j} - 0.5 P_{i,j + d} + \beta P_{i, j + 2d}$. Now if $i - j = a d$, then $A_{i,j} = \alpha_0(k - a) - 0.5 (k - a + 1) + \beta (k - 2 + a) = 0$ else $P_{i,j} = P_{i,j + d} =P_{i, j + pd} = 0 \implies A_{i,j} = 0$.
	\item $j \in \{q+1, \ldots, d\}$: $A_{i,j} = \alpha_1 P_{i,j} + \phi P_{i,j + d}$. Now if $i - j = a d$, then $a = 1$, implying $A_{i,j} = \alpha_1(k - 1) + \phi k = 0$ else $P_{i,j} = P_{i,j + d} = 0 \implies A_{i,j} = 0$.
\end{itemize}
Hence $A = I$.

Similar to the previous lemma, we have $\lambda_{\max}(P^{-1}) =  \max( \alpha_0 + 0.5 + \beta, \alpha_1 + |\phi|, 1 + 0.5 + 0.5) = \max(2 \beta + 1, 1, 2) = 2$ as $\alpha_1 + |\phi| = 1$ and $\beta\leq 0.5$ which follows from $2k - p - 1 \geq 1$.
\end{proof}

\end{document}